\documentclass[10pt, a4paper]{article}

\usepackage{graphicx}
\usepackage{color}
\usepackage{amsmath}
\usepackage{amssymb}
\usepackage{amscd}
\usepackage{bbm}
\usepackage{tikz}
\usetikzlibrary{patterns}
\usepackage{hyperref}
\hypersetup{
    bookmarks=true,         
    colorlinks=true,       
    linkcolor=red,          
    citecolor=green,        
    filecolor=magenta,      
    urlcolor=cyan           
}

\usepackage[utf8]{inputenc}
\usepackage{amsthm}
\usepackage{bbm} 
\usepackage[linesnumbered,lined,boxed,commentsnumbered]{algorithm2e}

\usepackage{marginnote}

\newcommand{\inr}[1]{\bigl< #1 \bigr>}

\newcommand{\norm}[1]{\left\|#1\right\|}%
\newcommand{\tnorm}[1]{{\left\vert\kern-0.25ex\left\vert\kern-0.25ex\left\vert #1 
    \right\vert\kern-0.25ex\right\vert\kern-0.25ex\right\vert}}

\newcommand\eps{\epsilon}

\DeclareMathOperator*{\argmin}{argmin}

\DeclareMathOperator*{\rank}{rank}

\def\ds1{\textrm{1\kern-0.25emI}} 

\newcommand \E{\mathbb{E}}
\newcommand \R{\mathbb{R}}

\newcommand \cA{{\cal A}}
\newcommand \cB{{\cal B}}
\newcommand \cC{{\cal C}}

\newcommand \cE{{\cal E}}
\newcommand \cF{{\cal F}}
\newcommand \cG{{\cal G}}

\newcommand \cM{{\cal M}}

\newcommand \cO{{\cal O}}

\newcommand \cQ{{\cal Q}}
\newcommand \cR{{\cal R}}
\newcommand \cS{{\cal S}}

\newcommand \cU{{\cal U}}

\newcommand \cX{{\cal X}}
\newcommand \cY{{\cal Y}}

\newcommand \bE{{\mathbb E}}

\newcommand \bP{{\mathbb P}}

\newcommand \bR{{\mathbb R}}

\usepackage{enumitem}

\newcommand{\PP}{\mathbb{P}}
\usepackage[]{}

\DeclareMathOperator*{\Med}{Med}
\DeclareMathOperator*{\Sym}{Sym}
\DeclareMathOperator*{\VC}{VC}
\DeclareMathOperator*{\Tr}{Tr}

\DeclareMathOperator*{\Card}{Card}
\DeclareMathOperator*{\Vect}{Vect}
\DeclareMathOperator*{\Id}{Id}
\newcommand{\Ba}{B_\Sigma}
\usepackage{braket}
\usepackage{authblk}

\newtheorem{defi}{Definition}
\newtheorem{prop}{Proposition}
\newtheorem{coro}{Corollary}
\newtheorem{remark}{Remark}

\newtheorem{setting}{Setting}
\newtheorem{theo}{Theorem}
\newtheorem{lemma}{Lemma}

\newenvironment{lemmabis}[1]
  {\renewcommand{\thelemma}{\ref{#1}$'$}%
   \addtocounter{lemma}{-1}%
   \begin{lemma}}
  {\end{lemma}}

\begin{document}

\title{Robust subgaussian estimation with VC-dimension}
\author[1]{Jules Depersin \\ email: \href{mailto:jules.depersin@ensae.fr}{jules.depersin@ensae.fr} \\ CREST, ENSAE, IPParis. 5, avenue Henry Le Chatelier, 91120 Palaiseau, France.}

\date{}                     
\setcounter{Maxaffil}{0}
\renewcommand\Affilfont{\itshape\small}

\maketitle

\begin{abstract}
Median-of-means (MOM) based procedures provide non-asymptotic and strong deviation bounds even when data are heavy-tailed and/or corrupted. This work proposes a new general way to bound the excess risk for MOM estimators. The core technique is the use of VC-dimension (instead of Rademacher complexity) to measure the statistical complexity. In particular, this allows to give the first robust estimators for sparse estimation which achieves the so-called subgaussian rate only assuming a finite second moment for the uncorrupted data.
By comparison, previous works using Rademacher complexities required a number of finite moments that grows logarithmically with the dimension. With this technique, we derive new robust sugaussian bounds for mean estimation in any norm. We also derive a new robust estimator for covariance estimation that is the first to achieve subgaussian bounds without $L_4-L_2$ norm equivalence.
\end{abstract}


\section{Introduction} 
\label{sec:introduction_on_the_mean_vector_problem}

Robustness has been a classical topic in statistics since the work of Hampel \cite{MR0301858,MR0359096}, Huber(\cite{MR2488795, MR0161415}) and Tukey \cite{MR0133937}. Informally, we call an estimator robust if it behaves nicely (in some sense) even when the data are not independent and sub-gaussian. In this work we deal with two different kinds of robustness :

\begin{itemize}[label=-, itemsep=0pt]
\item \textbf{Robustness against outlier}: The classical setup in statistics is that the observations were generated independently from a given probabilistic model. The field of robust statistics aims to relax this strong assumption, as real datasets are typically exposed to some source of contamination. Robustness to outliers has first been described in \cite{MR2488795}: in this context, the dataset is ``close" to the ideal setup, but has been corrupted. The corruption of the data can be modeled in several ways.
A classical example is Huber's contamination model, where the data are i.i.d. with common distribution 
$$dP=(1-\epsilon)dP_i+\epsilon dP_o . $$
Here $P_i$ is the distribution of the inliers, and $P_o$ is the distribution of outliers (see  \cite{MR2488795}).
Another example is the so-called $\mathcal{O} \cup \mathcal{I}$ framework, presented in \cite{guillaume2017learning}, where outliers may not be independant from each other, nor independant from the other uncorrupted data. The model that we deal with in this work is even more general and is sometimes referred to as ``adversarial contamination" (see, for instance, \cite{MR3631028}). The samples are generated from the following process: First, $N$ samples are drawn from some unknown distribution. Then, an adversary
is allowed to look at the samples and arbitrarily corrupt an $\epsilon$-fraction of them. In this setup, not only can outliers be correlated to each other and to inliers, but inliers can also be correlated to one another (because the adversary can choose which original samples to keep and in doing so correlating the samples that he keeps, for instance only keeping the largest samples when they are real-valued). In recent years, designing outlier-robust estimators for high-dimensional settings has become a challenge in a number of applications from the analysis of biological datasets to data poisoning attacks \cite{Voro18}.

\item \textbf{Robustness against heavy-tailed data} : In this work, we are interested in estimators whose risks are controlled with very high probability \emph{without making either boundedness or gaussian assumptions on the data} (or any other strong concentration assumptions). Indeed, we want to avoid those assumptions that limit severely the applicability of the results.
Notice that it is not sufficient to bound the expected loss of an estimator, since high probability guarantees derived from such bound (using for instance Markov's inequality) may be quite weak when data comes from heavy-tail distribution.

Let us give a simple example: For univariate mean estimation, if we are given $N$ independant realizations of a random variable with mean $\mu$ and variance $\sigma^2$, the empirical mean $\hat{\mu}$ satisfies, with probability $1-\delta$, $ |\hat{\mu}-\mu| \leq \sigma/\sqrt{N\delta}$. This bound rapidly deteriorates for small values of $\delta$, but it is sharp for the empirical mean in general \cite{AIHPB_2012__48_4_1148_0}. 
We call informally \emph{robust to heavy tail} an estimator whose ``loss" is bounded with probability  $1-\delta$ by a quantity proportional to $\log(1/\delta)$ when $\delta\to 0$.
The definition of the loss may change depending on the problem at stack, in mean estimation for instance it is just the squared distance $|\hat{\mu}-\mu|^2$. 
We say that an estimator \emph{achieves a subgaussian rate} if it achieves the same upper bound on the loss, up to multiplicative constants, as the empirical risk  minimizer of a $N$-sample of i.i.d. gaussian variables, \emph{even when the real sample is heavy tailed}. 
In a sense, we want our estimator to be as good as if the data were Gaussian, even when it is not !

\end{itemize}

These two informal definitions are different even if somehow related. 
For this to make sense, let us consider now high dimension mean estimation. 
Let $X_1, ..., X_N \in \R^d$ denote $N$ i.i.d. random variables with mean $\mu$ and covariance matrix $\Id$. 
The estimators $\hat{\mu}$ described in \cite{minsker2015geometric} or in \cite{JMLR:v17:14-273} both achieve, with probability larger than $1- \delta$, 

\begin{equation}\label{fake}
\|\hat{\mu}-\mu\|^2 \leq \frac{d\log(1/\delta)}{N}, 
\end{equation}where $\|\cdot\|$ is the canonical Euclidean norm on $\R^d$. 
This bound is proportional to  $\log(1/\delta)$, thus the estimators are \emph{robust to heavy tails}, but they do not \emph{achieve a subgaussian rate}. Indeed, if $Z_1, Z_2, ..., Z_N$ are independent identically distributed Gaussian variables $\mathcal{N}(\mu, \Id)$, it follows from Borell-TIS's inequality (see \cite[Theorem~7.1]{Led01} or \cite[pages 56-57]{MR2814399}) that with probability at least $1-\delta$,
\begin{equation*}
\norm{\bar Z_N-\mu}_2 = \sup_{\norm{v}_2\leq 1}\inr{\bar Z_N-\mu,v}\leq \E \sup_{\norm{v}_2\leq 1}\inr{\bar Z_N-\mu,v} + \sigma \sqrt{2\log(1/\delta)},
\end{equation*}
where $\sigma =\sup_{\norm{v}_2\leq1} \sqrt{\E\inr{\bar Z_N-\mu, v}^2}$. It is straightforward to check that $\E \sup_{\norm{v}_2\leq 1}\inr{\bar Z_N-\mu,v}\leq \sqrt{d/N}$ and $\sigma=\sqrt{1/N}$, which leads to the rate in \eqref{eq:sub_gauss} (where $C$ is an absolute constant), which is called \emph{subgaussian rate} 
\begin{equation}\label{eq:sub_gauss}
\norm{\bar Z_N-\mu}_2^2 \leq C \left(\frac{d}{N} + \frac{ \log(1/\delta)}{N}\right):= C r_\delta.
\end{equation}

The main difference with \eqref{fake} is that, in the subgaussian rate, the complexity term $d$ and the failure-dependant factor $\log(1/\delta)$ are not multiplied, but added instead. 
This rate is shown to be deviation-minimax up to constant factors in \cite{lugosi2018nearoptimal}. 
Recently, the seminal paper \cite{lugosi2019sub} described the first estimator to reach the subgaussian rate only assuming finite second moment, soon followed by other works in that path (\cite{guillaume2017learning}, \cite{depersin2019robust}, \cite{Bartlett19}, \cite{hopkins2018sub}).

In this work, we show that the analysis presented in \cite{lugosi2019sub}, in \cite{guillaume2017learning}, \cite{lerasle2019lecture} or in \cite{lecu2017robust}, all based on the Median-of-mean principle and the use of Rademacher complexities, can be modified in order to achieve sub-gaussian rates for sparse problems assuming only bounded two-order moments. The method developed in \cite{lerasle2019lecture} or in \cite{lecu2017robust} requires data to have at least $\log(d)$ finite moments (where $d$ is the dimension of the space) in order to exploit the sparsity of the problem and offers no guarantees without that requirement, and to the date is the best known. We show that we can drop this condition by judiciously introducing VC-dimension in the different proofs, and exploit the sparsity of the problem \emph{with only two moments}. Classical approaches using local Rademacher complexities cannot achieve this type of subgaussian bounds under only a second moment assumption in this setup. Indeed,  as shown in the counter-example from Section~3.2.3 of \cite{chinot2018statistical}, local Rademacher complexities may scale like $d^{1/8}$ whereas the right Gaussian bound should be of the order of $\sqrt{\log d}$. Somehow the classical approach used so far does not capture the right statistical complexity of high-dimensional problems under low-dimensional structural assumptions and under only a second moment assumption. Our VC-dimension based approach allows to overcome this issue and to go beyond this $\log d$ subgaussian moments assumption that has appeared in all works on robust and subgaussian estimation in the high-dimensional framework \cite{lerasle2019lecture}.  We also show that this general technique can be easily replicated and give new robust estimators that achieve state-of-the-art bounds for different estimation tasks such as :
\begin{itemize}[label=-, itemsep=0pt]
    \item Regression, already studied in \cite{lecu2017robust} where our estimator's rate match the one from \cite{lecu2017robust}, and sparse regression where our estimator's rate is the first to match the one from \cite{lecu2017robust} \emph{with only two moments}.
    \item  Mean estimation with non-Euclidean norms, studied in \cite{lugosi2018nearoptimal}, where our analysis gives a different rate that is better for some norms.
    \item  Covariance estimation, studied in \cite{mendelson2018robust} under $L_4-L_2$ norm equivalence: we do not need this assumption with our analysis, thus we give the first subgaussian estimator without this assumption. 
    \item Robust low-rank covariance estimation, studied in \cite{Wei2017EstimationOT}. We give an estimators that achieves better bounds and is to our knowledge the first MOM estimator for that problem.
\end{itemize}

Using VC-dimension in mean estimation, we loose a nice dependence of the risk bounds in the covariance structure: our rates for (non-sparse) mean estimation depend on the ambient dimension $d$ instead of the effective rank $\Tr(\Sigma)/||\Sigma||_{op}$. 
In particular, the general approach does not generalize directly to infinite dimensional spaces.
In the last section, we show that this issue can be overcome if we have some knowledge on the covariance matrix. 


\section{Warm-up : MOM principle, VC-dimension and mean estimation}

We start with the mean estimation problem in $\R^d$ that illustrates our technique: the goal is to estimate the mean $\E[Y]$ of a random vector $Y$ in $\R^d$ given a possibly corrupted dataset of i.i.d. copies of $Y$. The precise setting is the following: 

\begin{setting}\label{setting:first}
Let $(Y_1, ...,  Y_N)$ denote $N$ independent and identically distributed random vectors in $\mathbb{R}^d$.
We want to estimate $\E(Y_1)= \mu$, assuming that $Y_1$ has finite second moment.
Let $\Sigma=\E((Y_1-\mu) (Y_1-\mu)^T)$ denote the unknown covariance matrix of $Y_1$. 

The vectors $Y_1,\ldots, Y_N$ are not observed, instead, this dataset may have been corrupted, and this corruption may be adversarial: there exists a (possibly random) set $\mathcal{O}$ such that, for any $i \in \mathcal{O}^c, \ X_i = Y_i$. 
The $\mathcal{O}$ satisfies $|\mathcal{O} | \leq \left\lfloor \varepsilon N\right\rfloor$ 

The observed dataset is $\{X_i : i =1, ...,  N \} $, and we want to recover $\mu$. 
 \end{setting}
 
 Notice that there are no assumption on the data $\{X_i, i\in \mathcal{O}\}$.
 In particular these may be dependent of $\{Y_i : i =1, ...,  N \} $, and the $\{X_i : i \in \mathcal{O}\}$ may have arbitrary dependence structure. 
 
 \subsection{VC-dimension}\label{sec:VC}
 
 We start this part by recalling some basic facts about VC-dimension that appear for instance in \cite{JMLR:v20:17-504}.
 
 \begin{defi}
Let $\cF$ be a set of Boolean functions on any space $\cX$. We say that a finite set $S \subset \cX$ is shattered by $\cF$ if, for every subset $B \subset S$, there exists $f \in \cF$ such that $S \cap f^{-1}(\{1\})=B$. 
We call VC-dimension of $\cF$ (and note $\VC(\cF)$) the largest integer $n$ such that there exists a set $S$ of cardinal $n$ that is shattered by $\cF$.
 \end{defi}
 
Whenever $E$ is a Euclidean space, we will sometimes abusively call VC-dimension of a set $ C \subset E $ and note $\VC(C)$  the VC-dimension of the set of half-spaces generated by the vectors of $C$:
 $$\VC(C)=\VC(\{x \in E \rightarrow \mathbf{1}_{\braket{x,v} \geq 0}, \  v \in C\}). $$
 
Let us recall some basic facts about VC dimensions.
 \begin{enumerate}
     \item $\VC(\R^d)=d+1$. More generally, if $F$ a set of real-valued functions in a $k$-dimensional linear space, then $\text{Pos}(F)=\{x \rightarrow \mathbf{1}_{f(x) \geq 0}, f \in F\} $ has VC-dimension $k+1$ (see for instance \cite{dudley1978}, Theorem 7.2).
     \item For a function $g:\cY\to\cX$ , if we note $\cF \circ g =\{ f \circ g \ | \ f \in \cF\}$, then we have $\VC(\cF \circ g) \leq \VC(\cF)$.
     \item For any $r>0$, $\VC(\{x \in E \rightarrow \mathbf{1}_{\braket{x,v} \geq r}, v \in C\}) \leq  \VC(C-C) \lesssim \VC(C)$, see Section \ref{s:Proof}.
     \item Sauer's Lemma \cite{SAUER1972145}: Let $\cF$ denote a set a functions with VC-dimension $\nu$ and let $S$ be a set of $n\geq s$ points. Let $\cF * S=\{S \cap f^{-1}(\{1\}),  f \in \cF \}$, then 
     $$\Card(\cF * S)  \leq \left(\frac{e n}{\nu}\right)^\nu.$$
 \end{enumerate}
 
This last lemma can be used to prove the following result that is useful to bound the VC dimension of the set of sparse vectors.
 
 \begin{lemma} \label{sparseVc1}
 Let $\cF_1$, ..., $\cF_n$ denote $n$ sets of boolean functions, each having VC-dimension $\leq \nu$.
 Then, $$\VC(\cF_1 \cup \cF_2 \cup... \cup \cF_n) \leq 2\nu+ 2\log(n).$$
 \end{lemma}
 
Lemma~\ref{sparseVc1} is a straightforward extension of Theorem 3 in \cite{JMLR:v20:17-504}.
 
 \begin{proof}
 Let $S$ be a set shattered by $\cF= \cF_1 \cup \cF_2 \cup... \cup \cF_n$, and $s=\Card(S)$. Because $S$ is shattered, we have $\Card(\cF * S)= 2^s$. But we also have $ \cF * S = \cF_1 *S \cup \cF_2 *S \cup ... \cup \cF_n *S$, so $\Card(\cF * S) \leq n \left(\frac{e s}{\nu}\right)^\nu $.
 
It follows that $2^s \leq n \left(\frac{e s}{\nu}\right)^\nu $ or $s \leq \nu \log_2(es n^{1/\nu}/\nu)$.  By technical Lemma 4.6 in \cite{10.5555/524030} if $x \leq a\log_2(bx)$, then $x \leq 2 a\log_2(ab)$.
Hence, $s \leq 2\nu\log_2(en^{1/\nu})$, which implies Lemma~\ref{sparseVc1}. 
 \end{proof}
 
 \begin{coro} \label{sparseVC}
Fix $v_1, ..., v_d \in \R^n$ and note $\cU_s = \{\sum_i \lambda_i v_i \ | \  \lambda_i \in \R \ \& \ \sum_i \mathbf{1}_{\lambda_i \neq 0 } \leq s \}$ the set of s-sparse vectors, then 
 $$\VC(\cU_s) \leq 4 s \log_2(ed/s).$$
 \end{coro}
 
 To prove Corollary~\ref{sparseVC}, just write the set $\cU_s$ as a union of $\binom{d}{s}$ $s$-dimensional subspaces.
 As a side remark, we note that \cite{JMLR:v20:17-504} also shows that this bound is tight up to multiplicative constants: there exists an absolute constant $c$ such that $\VC(\cU_s) \geq c s \log_2(ed/s)$. 
 Besides, the result holds even if the set of vectors $(v_1,..., v_d)$ is not an orthogonal family or if it is not a base. 
Let us now recall an important theorem that will be very useful in regression and covariance estimation. 
Let $P=\{P_l,...,P_m\}$ denote a set of multivariate polynomials.
A sign assignment is an element $s$ of $\{+, -\}^m$.
The sign assignment $s$ is consistent with $P$ if there exists $x \in \R^n$ such that $P_i(x)\geq 0 \Leftrightarrow s_i=+ $.
 
 \begin{theo}[Warren, \cite{10.2307/1994937}]
 Let $P=\{P_1,...,P_m\}$ denote a set of polynomials of degree at most $\nu$ in $n$ real variables with $m > n$, then the number of sign assignments consistent for $P$ is at most $(4e\nu m/ n)^n$.
 \end{theo}
 
 \begin{coro} \label{polyVC}
Assume that the set of functions $\cF$ can be written $\cF=\{P \in \R^n_\nu[X]\rightarrow 1_{P(x)\geq 0}, x \in \R^n\}$, then $\VC(\cF) \leq 2 n \log_2(4 e \nu)$.
 \end{coro}
 
The following example will be useful in some applications.

\begin{prop} \label{important}
Let $r \geq 0$ and call $\cM_d^k(\R)$ the set of rank $k$, symmetric, $d$-dimensional matrices. 

Let $\cF=\{M \in \cM_d(\R) \rightarrow 1_{\braket{X,M}\geq r}, X \in \cM_d^k(\R) \}$. 
Then $\VC(\cF)=\VC(\cM_d^k(\R)) \leq 2 (d+1) k \log_2(12 e) $.
\end{prop}

\begin{proof}
Any $X \in \cM_d^k(\R)$ can be written $X= \sum_{i=1}^k \lambda_i \ x_i x_i^T$, with $(\lambda_i, x_i) \in \R \times \R^d$. 
Besides, for any $M$, the function $(\lambda_i, x_i)_{i \leq k} \rightarrow \braket{X,M} - r $ is a polynomial of degree $3$ in $k(d+1)$ variables. 
Hence, the result follows from Corollary \ref{polyVC}.
\end{proof}

Combining Lemma \ref{sparseVc1} applied with rank one matrices of the form $x x^T, x \in \R^s \times \{0\}^{d-s} $ and Corollary \ref{polyVC} yields the following result.

\begin{prop} \label{lessimportant}
Let $\cF=\{M \in \cM_d(\R) \rightarrow 1_{\braket{x x^T ,M}\geq r}, x \in \cU_s \}$. Then $\VC(\cF) \leq 16 s \log_2(ed/s) .$
\end{prop}

 \subsection{Median-of-mean} \label{MOM}

This work uses the median-of-means (MOM) approach which was introduced in  \cite{MR702836,MR1688610,MR855970} and has received a lot of attention recently in the statistical and machine learning communities  \cite{MR3124669,LO,MR3576558,MS,minsker2015geometric}. 
This  approach allows to build estimators that are robust to both outliers and heavy-tail data in various settings  \cite{MR1688610,MR855970,MR762855}. 
It can be defined as follows: we first randomly split the data into $K$ blocks $B_1,\ldots ,B_K$ of equal-size $m$ (if $K$ does not divide $N$, we just remove some data). We then compute the empirical mean within each block: for $k=1,\ldots,K$,
\begin{equation*}
 \bar{X}_k=\frac{1}{m}\sum_{i \in B_k} X_i.
 \end{equation*}  
In the one-dimensional case, the final estimator is the median of the latter $K$ empirical means.
This estimator has subgaussian deviations as shown in \cite{MR3576558}. 
The extension of this result to higher dimensions is not trivial as there exist several possible generalizations of the one dimensional median \cite{minsker2015geometric}. 

For any $k\in\{1,\ldots,K\}$, let $\textbf{X}_k := (X_i)_{i\in B_k}$ and $\textbf{Y}_k := (Y_i)_{i\in B_k}$. 
We start with a basic observation. 

\begin{remark} \label{bravo}
When $K\geq |\cO|$, there is at least $K-|\cO|$ blocks $B_k$ on which $\textbf{X}_k= \textbf{Y}_k$.
\end{remark}

For instance, if $K\geq 4 |\cO|$, then, there exist at least three quarters of the blocks $B_k$ where $\textbf{X}_k= \textbf{Y}_k$. 
We can now state the main lemma.

\begin{lemma}\label{main}
Let $\cF$ be a set of Boolean functions satisfying the following assumptions.

\begin{itemize}
    \item For all $f \in \cF$, $\PP\left(f(\textbf{Y}_1)=0\right) \geq 15/16$.
    \item $K \geq  C (\VC(\cF) \vee |\cO|)$ where $C$ is a universal constant.
\end{itemize}

 Then, with probability $\geq 1-\exp(-K/128)$, for all $f \in \cF$, there is at least $3K/4$ blocks $B_k$ on which $f(\textbf{X}_k)=0$.
\end{lemma}

In words, if each property $f$ is true for one non corrupted block with constant probability (here $15/16$ but it could be any fixed constant $\alpha \geq 1/2$) and $K$ is large enough, then, with very high probability, all properties are ``true for most of the blocks". 
This result is an alternative to \cite[Theorem 2]{lugosi2018nearoptimal} where the complexity is measured with VC-dimension instead of the Rademacher complexity. 
We show below that this difference yields to substantial improvements in some examples such as sparse multivariate mean estimation compared with the bounds in \cite{lugosi2018nearoptimal}. 
The strength of this result is that it is uniform in $\cF$ and gives an exponentially low failure probability, but its proof is quite simple. The proof of this result is given in Section~\ref{GenMeth}.

Clearly, the fraction $3/4$ of the block is arbitrary in Lemma~\ref{main}. 
In fact, up to some modifications of the constants, the same result holds for any \emph{fixed} fraction $\alpha <1$.

\subsection{Mean estimation} \label{mean_part}

Let $\norm{\cdot}$ denote a norm on $\R^d$ and let $\norm{\cdot}_*$ denote its dual norm. 
Let $\cB$ denote the unit ball for the norm $\norm{\cdot}$ and $\cB^*$ the one for the norm $\norm{\cdot}_*$. 
Let $\cB^*_0$ denote the set of extremal vectors of $\cB^*$. 
Let $\tnorm{A}= \sup_{u \in \cB^*}\norm{Au}_2$ where $\norm{\cdot}_2$ is the Euclidean norm on $\R^d$. Let

$$ \hat{\mu}_K= \argmin_{a \in \R^d} \  \max_{ u \in \cB^*_0  }  \ \Med(\braket{\bar{X}_k-a, u}). $$

\begin{theo}\label{main2}
There exists an universival constant $C$ such that if $K \geq C (\VC(\cB^*_0) \vee |\cO|) $, then, with probability larger than $1-\exp(- K/128)$,
 $$\|\hat{\mu}_K -\mu \| \leq 8 \ \tnorm{\Sigma^{1/2}}\sqrt{\frac{K}{N}}.$$
\end{theo}

In particular, for any $\delta \in [e^{-c N},1/2]$, there exists an estimator $\mu_\delta$ such that 

\begin{equation}\label{eq:pq}
\norm{\mu-\mu_\delta} \lesssim \tnorm{\Sigma^{1/2}}\left(\sqrt{\frac{\VC(\cB^*_0)}{N}}+  \sqrt{\frac{\log(1/\delta)}{N}}+ \sqrt{\epsilon}\right).
\end{equation}\newline

The 'outlier' term $\sqrt{\epsilon}$ can be shown to be optimal in important cases (see the remarks after \cite[Theorem 1.3]{MR3909640}).
The deviation term ($\tnorm{\Sigma^{1/2}}\sqrt{\log(1/\delta)/N}$) is the same as in the Borel-TIS inequality, $\tnorm{\Sigma^{1/2}}$ being the weak variance term.
It is optimal as shown in \cite{lugosi2018nearoptimal}. 
The difference is the complexity term, which is here $\tnorm{\Sigma^{1/2}}\sqrt{\VC(\cB^*_0)/N}$. Neither \cite{lugosi2018nearoptimal} nor this work build estimators achieving the subgaussian rate, where this complexity is $\E(\norm{G})/\sqrt{N}$, $G$ being a centered Gaussian vector with the same covariance as $Y$. We do not know if this rate can be reached in general. 

\begin{remark}
The inequality $\VC(\cB^*_0) \leq d+1$ gives a general bound on the complexity term.
\end{remark}

The complexity term in \cite{lugosi2018nearoptimal}, which can also be found in \cite[Chapter 4, Lemma 47]{lerasle2019lecture} is $\E(\|\tilde{Y}\|)/N$ where $\tilde{Y}=\sum \epsilon_i (Y_i-\mu)$, $\epsilon_i$ being i.i.d. Rademacher variables. 
Here it is $\tnorm{\Sigma^{1/2}}\sqrt{\VC(\cB^*_0)/N}$.
Which of them is the best depends on the situation. 
For instance, when one wishes to estimate with respect to $\norm{\cdot}_2$, the Euclidean norm on $\R^d$, $\E(\|\tilde{Y}\|_2)/N \simeq \sqrt{\Tr(\Sigma)/N} $, while $\tnorm{\Sigma^{1/2}}\sqrt{\VC(\cB^*_0)/N}= \sqrt{\lambda_1 d/N}$, $\lambda_1$ being the largest eigenvalue of $\Sigma$, so the former is better. 
In this example, the bound in VC dimension looses the nice dependence in the covariance structure. 
On the other hand, suppose that we want to estimate $\mu$ with respect to the sup norm $\norm{a}_{\infty}=\max\{a_1, ..., a_n\}$ and assume that $\Sigma=\Id$ for simplicity.
Then $\tnorm{\Id}=1$ and $ \VC(\cB^*_0) \lesssim \log(d)$ so
$$\tnorm{\Sigma^{1/2}}\sqrt{\VC(\cB^*_0)/N} \simeq \sqrt{\log(d)/N}. $$ 
On the other hand, if we only have two moments on the coordinates of $Y$, then the best bound on the Rademacher complexity is $\E(\|\tilde{Y}\|_2)/N$ which is of order $\sqrt{d/N}$ in general (to see that, take for $Y_1$ a random vector whose coordinates are independent, equal to $\sqrt{d N}$ with probability $1/(dN)$ and $0$ otherwise).  

\begin{remark}
The analysis of Section~\ref{GenMeth} and in particular Lemma \ref{general}', shows that the estimator $\hat{\mu}_K$ achieves the bound $\tnorm{\Sigma^{1/2}}\sqrt{K/N}$ when $K\geq\cC\vee |\cO| $, where the complexity $\cC$ is the minimum between the VC dimension $\VC(\cB_0^*)$ and the Rademacher complexity $\E(\|\tilde{Y}\|)^2/(N \tnorm{\Sigma^{1/2}})$. 
Therefore both our bounds and the bound of \cite{lugosi2018nearoptimal} hold simultaneously and we can always keep the ``best complexity term" among VC and Rademacher complexity. 
As the main novelty here is the introduction of the VC-dimension, we do not remind this fact in each application.
The interested reader can have in mind that, in most examples, the same result holds and the estimators have risk bounds smaller than both complexities. Our aim is to show that VC type bounds are particularly efficient in $L_2$ and sparse scenarii, when Rademacher complexity fails to achieve optimal bounds.
\end{remark}

 
\section{Sparse setting and other estimation tasks} 

This section shows that the methodology of Theorem \ref{main2} also applies to a great variety of estimation tasks. 
Let us start with the example of sparse mean estimation for the Euclidean norm.

\subsection{Sparse mean estimation} \label{sparse_part}

For any $v_1, ..., v_d \in \R^n$, let $\cU_s(v_1, ..., v_d) = \{\sum_i \lambda_i v_i \ | \  \lambda_i \in \R \ \& \ \sum_i \mathbf{1}_{\lambda_i \neq 0 } \leq s \}$ denote the set of s-sparse vectors over the dictionary $\{v_1, ..., v_d\}$.
We fix for this part the vectors $v_1, ..., v_d$ and we note $\cU_s=\cU_s(v_1, ..., v_d)$. 
We consider Setting~\ref{setting:first} and assume furthermore that $\mu$ belongs to $\cU_s$. 
We note $\cB_2$ the unit ball for the canonical Euclidean norm in $\R^n$, and we propose the estimator 

$$ \hat{\mu}_K= \argmin_{a \in \mathcal{U}_s} \  \max_{u \in \mathcal{U}_{2s} \cap \cB_2 }  \ \Med \braket{\bar{X}_k-a, u}.$$

 \begin{theo} \label{sparseTheo} 
 There exists an absolute constant $C$ such that, if $K \geq C ( s \log(d/s) \vee |\cO|) $, then, with probability larger than $1-\exp(- K/128)$,
 $$\norm{\hat{\mu}_K -\mu}_2 \leq 8  \sqrt{\frac{\lambda_1(\Sigma) K}{N}}.$$
Here, $\lambda_1(\Sigma)$ is the largest eigenvalue of $\Sigma$.
 \end{theo}
 
The conclusion of Theorem~\ref{sparseTheo} can be writen as follows.
For any $\delta \in [e^{-c N},1/2]$, there exists an estimator $\mu_\delta$ such that 

\begin{equation*}\label{eq:intro_subgaus_rate}
\|\mu-\mu_\delta\|_2 \lesssim \lambda_1(\Sigma)(\sqrt{\frac{ s \log(d/s)}{N}}+  \sqrt{\frac{\log(1/\delta)}{N}}+ \sqrt{\epsilon}).
\end{equation*}

We see that the complexity ($s \log(d/s)$) is once again decoupled from the deviation ($ \log(1/\delta)$), which is not the case in works such as \cite{JMLR:v17:14-273} where those two terms are multiplied together. The complexity term $s \log(d)/N$ is not optimal because it does not depend on the structure of $\Sigma$ (see Section~\ref{last} for details). 
However, our complexity term is interesting for two main reasons: 
\begin{itemize}
    \item This is the first sparsity dependent bound that holds without higher moments conditions than the $L_2$ ones.
    By contrast, \cite{lerasle2019lecture} or \cite{lecu2017robust} need to assume the existence $\log(d)$ subgaussian moments in order to make the sparsity appear, and offer no guarantees without that requirement.
    \item It comes close to the theoretic optimal when $\Sigma \simeq \lambda \Id$.
\end{itemize}

\subsection{Regression}
In this section, we consider the standard linear regression setting where data are couples $(Y_i, V_i)_i \in \R^d \times \R$ and we look for the best linear combination of the coordinates of $Y_i$ to predict $V_i$, that is we look for $\beta^*$ defined as follows: Given $\cS \subset \R^d$ (in practice, we will only study $\cS =\R^d$ or $\cS=\cU_s$),
$$\beta^*=\argmin_{\beta \in \cS}l(\beta)=\argmin_{\beta \in \cS} \E(V_1-\braket{\beta , Y_1})^2 .$$ 
As in the previous section, the observed dataset $(X_i, Z_i)_i \in \R^d \times \R$ is a corrupted version of the i.i.d. dataset $\{(Y_i, V_i)_i, i\in\{1,\ldots,N\}\}$ in a possibly adversarial way.
The assumptions made on good data $(Y_i, V_i)_i$ are gathered in the following setting: (see also \cite{lerasle2019lecture} or \cite{audibert2011}).
  
 \begin{setting}
There exists a (possibly random) set $\cO$ such that, for any $ i \in \cO^c, \ (X_i, Z_i)=(Y_i, V_i)$, where $(Y_i, V_i)$ are independent identically distributed observations in $\R^d \times \R$. Let $\xi_i=V_i-\braket{\beta^* , Y_i}$, we make four main assumptions:
 \begin{itemize}
     \item $Y_1$ has finite second moment and write its $L^2$-moments matrix $\Sigma=\E(Y_1 Y_1^T)$. Let also $\Ba=\{x \in \cS- \cS| \braket{x,\Sigma x} \leq 1\}$ be the ellipsoid associated with this $L_2$ structure. 
     \item Let $\sigma^2 := \sup_{u \in \Ba} \E(\xi_1^2 \braket{u , Y_1}^2)$ and assume that $\sigma^2 < \infty$.
     \item There exists an universal constant $\gamma$ such that, for all $u \in \cS-\cS$, $\E(|\braket{u,X}|) \geq \gamma \sqrt{\E(|\braket{u,X}|^2)}$.
     \item $\E(\xi_1 Y_1)=0$
 \end{itemize}
 \end{setting}
Condition 2 is implied by Assumptions 3.5 and 3.7 in \cite{audibert2011}, the same assumption is made in \cite{lerasle2019lecture}.

Condition 3 is called the ``small ball hypothesis", it is described in details in \cite{mendelson2017extending} or in \cite{lecu2016regularization} for instance. It is implied by Condition 3.5 in \cite{audibert2011}, it is stated similarly in \cite{lerasle2019lecture}.

Condition 4 is always true in the non sparse-case. In the sparse case, it is true in a number of applications, for instance, the very important when the noise $\xi$ and $Y$ are independant.

The two last conditions may seem exotic, we refer to \cite[Section 3]{audibert2011}
 for detailed discussions and examples where these are satisfied. For the moment, we may emphase that they involve only first and second moment conditions on $\xi_1$ and $\braket{u,Y_1}$
 
  Our estimator is defined as follows: Let $\hat B_\Sigma = \{u \in \cS-\cS\ |\ \cQ_{1/4}^k \frac{1}{m}\sum_{i \in B_k}|\braket{u , X_i}|^2 \leq 1 \} $, where $\cQ_{1/4}^k$ is the first quartile over $k \leq K$ : for any sequence $x_1, ..., x_n \in \R$, if we note $x^*_1, ... , x^*_n$ the corresponding increasingly ordered sequence, $\cQ_{1/4}^k x_k= x^*_{\lfloor n/4\rfloor}$. Then, 
 $$\hat{\beta}= \argmin_{a \in \cS} \max_{u\in \hat B_\Sigma} \Med_k \sum_{i \in B_k}(Z_i-\braket{a , X_i})\braket{u , X_i} . $$
 This new estimator satisfies the following result.
 
 \begin{theo}\label{theo:Reg}
There exists an absolute constant $C$ such that the following holds.
Let $\cS =\R^d$ or $\cU_s$ and let $N \gamma^2/64 \geq K\geq C(\VC(\cS-\cS) \vee |\cO|)$. 
Then, with probability $\geq 1-\exp(-K/128)$,
 
$$ \braket{\hat \beta- \beta^*, \Sigma (\hat\beta-\beta^*)} \leq 128  \ \sigma^2  \frac{K}{N \gamma^4}.$$
 \end{theo}

For all $\beta$, 
\[
l(\beta) = l(\beta^*)+ 2 \E(\xi_1 \braket{\beta-\beta^* , Y_1})+ (\beta-\beta^*) \Sigma (\beta-\beta^*)\leq l(\beta^*) + (\beta-\beta^*) \Sigma (\beta-\beta^*).\]
So, if $r=\sqrt{128}  \ \sigma  \sqrt{\frac{K}{N}} $ , then 
\[
l(\hat\beta)-l(\beta^*) \leq  r^2/\gamma^4.
\]

The conclusion of Theorem~\ref{theo:Reg} can be written as follows: for any $\delta \in [e^{-c N},1/2]$, there exists an estimator $\mu_\delta$ such that 

\begin{equation*}\label{eq:intro_subgaus_rate}
\braket{\hat \beta- \beta^*, \Sigma (\hat\beta-\beta^*)} \lesssim \sigma (\sqrt{\frac{ \VC(\cS)}{N}}+  \sqrt{\frac{\log(1/\delta)}{N}}+ \sqrt{\epsilon}).
\end{equation*}

Once again we notice the nice decoupling between complexity and deviation. This result is interesting for a several reasons, the main one being that this work is the first that gives a bound holding with exponential probability, that holds without assuming more than 2 moments on the design $Y_1$, even in the sparse setting. 
By comparison, \cite{lecu2017robust} or \cite{lugosi2017regularization} for instance, assume that at least $\log(d)$ subgaussian moments exist to achieve this kind of rate and offers no guarantees without that requirement and are the best to the date.

\subsection{Covariance estimation}

 This section studies the problem of robust covariance estimation. 
Consider Setting \ref{setting:first}, and assume that $\mu$ is known, fixed to $0$ without loss of generality.
We want to estimate $\Sigma$. This problem has a number of applications: the bounds we present can for instance easily be transposed (with the Davis–Kahan theorem) to the problem of robust PCA. 
It has already been studied in \cite{Wei2017EstimationOT}, or \cite{JMLR:v17:14-273}, but these estimators do not exhibit any decoupling between complexity and deviation. In \cite{mendelson2018robust}, the authors propose a robust estimator for covariance using the MOM method, and get the optimal complexity-deviation decoupling. They also give interesting comments and insights about this estimation problem. However, they need a $L_4-L_2$ norm equivalence in order to do so. In addition, they do not study the problem of low rank estimation that we present here. 
 
For any matrix $A$, define its spectral norm by 
$$\tnorm{A}=\sup_x\frac{\norm{Ax}_2}{\norm{x}_2}.$$
Let $\Sym(d)$ denote the set of $d$ dimensional symmetric positive matrices.
Assume that 
$$\sigma^2= \sup_{u\in \cB_2} \E\left(\braket{u,(\Sigma-Y_1 Y_1^T) u}^2\right) < \infty.$$

 This quantity is sometime refered to as \emph{weak variance} of a random matrix \cite{mendelson2018robust}. Our estimator is defined as follows 
 $$\hat{\Sigma}=\argmin_{M \in \Sym_d} \sup_{\norm{u}_2=1} \Med_k \braket{u,(\frac{1}{m}\sum X_i X_i^T-M) u}.$$
 It satisfies the following bound.
 \begin{theo} \label{theo:Cov1}
There exists an absolute constant $C$ such that, if $K \geq C (d \vee |\cO|) $, then, with probability larger than $1-\exp(- K/128)$,
 $$\tnorm{\hat{\Sigma}_K -\Sigma} \leq 8 \ \sigma \sqrt{\frac{K}{N}}.$$
 \end{theo}

 \begin{coro} Assume that $R = \sup_u \frac{\sqrt{\E\braket{u,Y})^4}}{\E\braket{u,Y}^2} < \infty$, then, for $K \geq C (d \vee |\cO|) $
   $$\tnorm{\hat{\Sigma}_K -\Sigma} \leq 8 \ R  \tnorm{\Sigma} \sqrt{\frac{K}{N}}.$$
 \end{coro}
 
The ``bounded kurtosis assumption" $R<\infty$ appears similarly in \cite{JMLR:v17:14-273}.
In \cite{JMLR:v17:14-273}, the estimator achieves a bound of order $r(\Sigma)\tnorm{\Sigma} \sqrt{K/N}$ where $r(\Sigma)$ is the effective rank of the covariance matrix: once again, the complexity $r(\Sigma)\tnorm{\Sigma}$ is multiplied by the deviation term $K\propto \log(1/\delta)$ in this case, while here they are decoupled : the dimension does not multiply $K$ in our bound. In \cite{mendelson2018robust}, authors give a better rate (because they only need $K$ to be larger than $r(\Sigma)$ instead of $d$ for the bound to hold) but they use the $L_4-L_2$ norm equivalence, which is a strong hypothesis we do not need here. \\

To conclude this section, we propose an estimator for covariance estimation for the Frobenius norm. 
What is interesting here is that, when we know that the covariance matrix has rank $r$, we can exploit the low-rank assumption in the VC-dimension analysis, while it might be harder to grasp using other forms of complexity. 
This work is the first, to our knowledge, to exhibit the decoupling between complexity and deviation for this problem, and this estimator is the first of its kind. 
We note $\norm{\cdot}_F$ the Frobenius norm for matrices, and $\braket{\cdot,\cdot}_F$ the inner product for matrices. 

\begin{theo}\label{theo:Cov2}
Let
$$\hat{\Sigma}=\argmin_{M \in \cM_d^r(\R)} \ \ \sup_{ U \in \cM_d^{2r}(\R), \ \norm{U}_F=1} \Med_k \braket{\frac{1}{m}\sum X_i X_i^T-M),U}_F.$$
There exists an universal constant $C$ such that, if $\rank(\Sigma)\leq r$ and $K \geq C (d \times r \vee |\cO|) $, then, with probability larger than $1-\exp(- K/128)$,
 $$\norm{\hat{\Sigma}_K -\Sigma}_F \leq 8 \ \sigma \sqrt{\frac{K}{N}}.$$
\end{theo}

The conclusion of Theorem~\ref{theo:Cov2} can be written as follows: for any $\delta \in [e^{-c N},1/2]$, there exists an estimator $\Sigma_\delta$ such that 

\begin{equation*}\label{eq:intro_subgaus_rate}
||\Sigma_\delta-\Sigma||_2 \lesssim \sigma (\sqrt{\frac{ r d }{N}}+  \sqrt{\frac{ \log(1/\delta)}{N}}+ \sqrt{\epsilon}).
\end{equation*}

Compared with \cite{JMLR:v17:14-273}, we note that this bound on the risk of our estimator does decouple the complexity $ r d $ and the deviation $ \log(1/\delta)$ as announced. Again it holds under only a $L_2$ moment assumption, with no extra subgaussian moment needed to get this bound.

\section{An algorithm to improve risk bounds} \label{last}
\label{sec:SDP}

The different applications of Lemma~\ref{main} show that, in general, the complexity term derived from this result is not optimal.
For example, for mean estimation estimation in Euclidean norm, the complexity term reached by our estimator is proportional to $\sqrt{\lambda_1(\Sigma)d/N}$, where the best rate would be $\sqrt{\Tr(\Sigma)/N}$. 

In this section, we provide an algorithm that leads to better and in some cases optimal complexity rates.  
The price to pay is that these new estimators require some knowledge on the covariance matrix $\Sigma$. 
We will consider the example of sparse mean estimation for the sake of clarity, but we argue that it also holds for sparse (and non-sparse) regression.
We therefore use the setting \ref{setting:first} and assume furthermore that the mean $\mu$ belongs to $\cU_s(v_1,...,v_d)$, where the set of vectors $(v_1,...,v_d) \in \R^d$ is fixed and known. \\

Let $\lambda_1, \lambda_2, ..., \lambda_d$ denote the eigenvalues of $\Sigma$ in decreasing order, and let $e_1, ..., e_d$ denote a set of normalized corresponding eigenvectors. 
For any $ 1 \leq n <  \lfloor \log_2(d) \rfloor :=n_{l}$, let $s_n$ denote the largest index such that $\lambda_{s_n} \geq \lambda_1/ 2^n$.
In particular, $\lambda_1 \geq ... \geq \lambda_{s_1}\geq  \lambda_1/2 > \lambda_{s_1+1}...$.
By convention, let $s_0=0$. 
Finally, we note $E_n=\Vect\{e_{s_{n-1}+1}, e_{s_{n-1}+2},...,  e_{s_n} \}$, with convention $E_{n_{l}}=\Vect(e_{s_{n_{l}-1}+1}, e_{s_{n_l-1}+2},...,  e_{d})$. 
If we know the matrix $\Sigma$, we can identify the eigenspaces $E_n$ and thus compute the orthogonal projections of the data on these subspaces: $X_k^i: =\mathbf{proj}_{E_i}(X_k)$, for $i \in \{1, ..., n_l\}$ and $k \in \{1, ..., K \}$.

In Section~\ref{sparse_part}, we described a procedure that takes as input an integer $K \geq c_0 (\tilde{s} \log(\tilde d/\tilde{s}) \vee |\cO|) $ and a (possibly corrupted)  dataset $Z_1, ... Z_N$ having common mean $\tilde \mu$ which is $ \tilde{s}$-sparse relatively to a set of vectors $(u_1, u_2, ...,  u_{\tilde{d}} )$ and common covariance matrix $\tilde \Sigma$.
The procedure returns $\hat \mu_K$ satisfying, with probability at least $1-\exp(-c_1 K)$
 $$\|\hat{\mu}_K -\mu \|_2 \leq 8  \sqrt{\frac{\tnorm{\tilde \Sigma} K}{N}}.$$ 
  
Let $\texttt{proc}(Z_1, ..., Z_N, K, u_1, u_2, ..., u_{\tilde{d}}, \tilde{s})$ denote the output of this procedure. 
The idea of the algorithm is to project on the subspaces $E_i$ and apply this preliminary procedure on those subspaces. 
Let $d_i$ the dimension of $E_i$. 
The algorithm is formally defined as follows:
     
  \vspace{0.7cm}
 \begin{algorithm}[H]\label{algo:almost_final}
\SetKwInOut{Input}{input}\SetKwInOut{Output}{output}\SetKw{Or}{or}
\SetKw{Return}{Return}
\Input{$X_1, \ldots, X_N$ and $K \geq |\cO|$. }
\Output{A robust subgaussian estimator $\hat \mu_\delta$.}  
\BlankLine
$i \leftarrow 1$.\\
\While{ $i\leq n_l$}{
Compute $X_1^i, ..., X_n^i$. \\
Compute $u_1^i, u_2^i, ... , u_{d}^i $ the orthogonal projections of $v_1, v_2, ... v_{d}$ onto $E_i$. \\
$K_i \leftarrow K \times 2^{i-1}$.\\
\pagebreak
\If{$d_i < s \log(d/s)$}{$\mu_i\leftarrow \texttt{proc}(X_1^i, ..., X_n^i, K_i, e_{s_{i-1}+1}, ..., e_{s_{i}}, d_i)$.\\}
\Else{$\mu_i\leftarrow \texttt{proc}(X_1^i, ..., X_n^i, K_i, u_1^i, u_2^i, ...,  u_{d}^i ,s)$.\\}
$i \leftarrow i+1$.}
$\hat \mu_K \leftarrow \sum_{j\leq i } \mu_i $.\\
\Return $\hat \mu_K$.
 \caption{Pseudo-code of the robust sub-gaussian estimator of $\mu$}
\end{algorithm}
\vspace{0.7cm}

The algorithm produces an estimator $\hat \mu_K$ satisfying the following result.
  
\begin{prop}\label{prop:BetterRB1}
Assume setup \ref{setting:first}. There exists an absolute constant $C$ such that, if $K \geq  C [\sum_i 2^{-i} (d_i \wedge s \log(d/s) ) \vee |\cO|]$, the output $\hat{\mu}_K$ of Algorithm \ref{algo:almost_final}, satisfies, with probability $\geq 1-  2\exp(-K/128)$,
$$\|\hat{\mu}_K -\mu \|_2 \leq 8 \log(d)  \sqrt{\frac{\lambda_1 K}{N}}.$$
\end{prop}

The complexity term in Proposition~\ref{prop:BetterRB1} is better than in Theorem~\ref{sparseTheo}, because $\sum_i 2^{-i} (d_i \wedge s \log(d/s) ) \leq s \log(d/s)$. 
More importantly, this complexity term depends on the the covariance structure of the data through the $d_i$. 
In the case of sparse mean estimation, we can deduce a precise estimate of this complexity term: for any $\delta \in [e^{-c N},1/2]$, there exists an estimator $\hat\mu_\delta$ such that:
\begin{equation*}
\norm{\hat \mu_\delta - \mu}_2 \leq   C \log(d) \left( \sqrt{\frac{\log(d/s)\sum_{i=1}^s \lambda_i}{N}} +  \sqrt{\frac{\lambda_1 \log(1/\delta)}{N}}+\sqrt{\lambda_1 \epsilon } \right).
\end{equation*}

This estimate comes from the bounds $\sum_{i=1}^s \lambda_i \geq \sum_{i\leq j+1}  \lambda_1 2^{-i-1} d_i \wedge s $, where $j$ is such that $\sum_{i\leq j} d_i\leq s \leq \sum_{i\leq j+1} d_i $. 
We then write $\sum_i 2^{-i} (d_i \wedge s \log(d/s) )\leq \log(d/s) \sum_{i\leq j+1} 2^{-i} d_i\wedge s +  2^{-j} s \log(d/s) \leq 4 \log(d/s) \sum_{i=1}^s \lambda_i  $. \\

Thies result is proved in Section~\ref{preuve:RB1}. 
We argue that the very same proof can be replicated for the regression problem, if  $\E(\xi_1^2 Y_1 Y_1^T) $ has the same eigenspaces as $\E( Y_1 Y_1^T) $.
This happens for instance when $\xi$ is independent of $Y$.  We end up with the bound of Theorem \ref{theo:Reg}, holding whenever $K \geq  C [\sum_i 2^{-i} (d_i \wedge \VC(\cS ) \vee |\cO|]$ instead of $K \geq C [\VC(\cS  \vee |\cO|] $


\section{Conclusion : concurrent work and discussion}

This work is not the first to deal with robust estimation: a lot of results and algorithms have been developed over the past few years for sparse estimation in presence of outliers  (see for instance \cite{li2017robust}) but most of these works assume that non-corrupted data are Gaussian. For instance \cite{MR3845006} already deal with mean and covariance estimation using extensions of the Tukey-depth (and using VC-dimension), but their methods rely on informative data being Gaussian. 

Robustness to heavy-tailed data has also been studied in various works, see \cite{lugosi2019mean} for a survey of recent developments. 
We already mentioned two articles that this work tries to complete and improve: \cite{lugosi2018nearoptimal} for mean estimation under any norm and \cite{lecu2017robust} for sparse regression.
Though the techniques involved are close, this work illustrates that using VC-dimension can drastically improve risk bounds in various applications, in particular in the sparse setting.\\

Concurrent work : After the initial submission of this manuscript, we became aware of two concurrent works \cite{prasad2019unified} and \cite{prasadFDP}. Authors use an approach based on the $1/2$-cover of the unit sphere to deal with mean and sparse-mean estimation for the euclidean norm (\cite{prasad2019unified}), and with covariance estimation for spectral norms (\cite{prasadFDP}). They get close-to-optimal bounds for those two problems, with a remaining extra-logarithmic term. They do not tackle mean estimation in any norms, regression or low-rank covariance estimation. \\

There are still many exciting open questions. 
The quantity that is crucial in all the studies is  $$\E\left(\sup_{f} \sum_{k=0}^K f(\textbf{Y}_k)-K\E(f(\textbf{Y}_k))\right)$$ where $f$ are boolean functions. 
In mean estimation for instance,  $\E\left(\sup_{v \in V} \sum_{k=0}^K \mathbf{1}_{\braket{\bar{Y}_k-\mu, v} \geq r}-K\E(\mathbf{1}_{\braket{\bar{Y}_k-\mu, v} \geq r})\right)$ is the important quantity. 
Bounding this quantity using the VC-dimension of $V$ yields a bound independant of the covariance of $Y$. 
On the other hand, bounding that quantity by the Rademacher complexity of the $Y_i$ (like in \cite{lecu2017robust},  \cite{lugosi2018nearoptimal} or here in Part \ref{GenMeth}) stating that  \[
\E(\sup_{v \in V} \mathbf{1}_{\braket{\bar{Y}_k-\mu, v} \geq r}-K\E(\mathbf{1}_{\braket{\bar{Y}_k-\mu, v}}))< K \frac{\sqrt{\E(\|\tilde Y\|)}}{r\sqrt{N}}
\]
does not exploit the boundedness of the indicator function and necessitates unnecessary stronger assumptions on data. 
The ideal would be to conciliate both ideas, and to find a nice in-between that would take into account both the boudedness and the dependency in the covariance structure.

 The last point we make is about computational issues. 
 The estimators presented can not be implemented as is. 
 Nevertheless, encouraging recent works have shown that "relaxed", computable estimators can be derived from this kind of work. For instance the pioneer work of \cite{hopkins2018sub}, followed by  \cite{depersin2019robust} and \cite{Bartlett19} for instance, derived tractable estimators, in polynomial times, from the work of \cite{lugosi2019sub}. Even more recently, some new tractable estimators for regression and covariance estimation with heavy-tailed data have emerged in \cite{cherapanamjeri2019algorithms}. We can hope for this work to be made tractable as well, which seems to be quite a challenge.

\section{Proofs} \label{s:Proof}

\subsection{A fact about VC-dimension}

For any Euclidean space $E$, and any $C \in E$

\begin{lemma}
$\VC(\{x \in E \rightarrow \mathbf{1}_{\braket{x,v} \geq r}, v \in C\}) \leq  \VC(C-C) \leq c_0 \VC(C) $ where $c_0$ is universal constant.
\end{lemma}

\begin{proof}
Assume that a set $x_1, ... , x_d \in E$ is shattered by $\cF=\{x \in E \rightarrow \mathbf{1}_{\braket{x,v} \geq r}, v \in C\}$. Then, for any $I \subset \{1,2, ..., d\}$, there is a vector $v_1$ so that $\braket{v_1,x_i} \geq r $ if and only if $i \in I $. There is a vector $v_2$ so that $\braket{v_2,x_i} < r $ if and only if $i \in I $. Then we have $\braket{v_1-v_2, x_i} \geq 0$ if and only if $i \in I $, so $\{x \in E \rightarrow \mathbf{1}_{\braket{x,v} \geq 0}, v \in C-C\}$ shatters $x_1, ... , x_d$, and  $\VC(\{x \in E \rightarrow \mathbf{1}_{\braket{x,v} \geq r}, v \in C\}) \leq  \VC(C-C)$. \\

Now we see that  $\VC(\{(x,y) \in E^2 \rightarrow \mathbf{1}_{\braket{x,v}+ \braket{y,w}  \geq 0}| (v,w) \in C\times C\}) \geq \VC(C-C)$ because if $x_1, ... , x_d \in E$ is shattered by $\{x \in E \rightarrow \mathbf{1}_{\braket{x,v} \geq 0}, v \in C-C\}$ then $((x_1, -x_1), ... , (x_d-x_d)) \in E\times E$ is shattered by $\{(x,y) \in E^2 \rightarrow \mathbf{1}_{\braket{x,v}+ \braket{y,w}  \geq 0}| (v,w) \in C\times C\}$. Theorem 1.1 in \cite{vandervaart} states that $\VC(C \times C) \leq c_0 \VC(C)$ for some constant $c_0$, and that concludes the proof.
\end{proof}

\subsection{General methodology} \label{GenMeth}
We begin by proving the main lemma \ref{main} of Part \ref{mean_part}

\begin{proof}
We want to prove that, with probability $\geq 1-\exp(-K/128)$,
$$\sup_{f} \sum_{k=0}^K f(\textbf{X}_k) \leq K/4.$$
If $C \geq 16$, $K \geq 16 |\cO|$ and it is sufficient to show that $\sup_{f} \sum_{k=0}^K f(\textbf{Y}_k) \leq 3K/16$ by Remark \ref{bravo}. 
Now we write
$$\sup_{f} \sum_{k=0}^K f(\textbf{Y}_k) \leq \underbrace{\sup_{f} \sum_{k=0}^K f(\textbf{Y}_k)-\E\left(\sup_{f} \sum_{k=0}^K f(\textbf{Y}_k)\right)}_{Deviation = D}+\underbrace{\E\left(\sup_{f} \sum_{k=0}^K f(\textbf{Y}_k)\right)}_{Magnitude=M}. $$

By the bounded difference inequality \cite[Theorem~6.2]{MR3185193}, with probability $\geq 1-\exp(K/128)$, $D \leq K/16$.

For the magnitude term, we write 
$$M \leq  \E\left(\sup_{f} \sum_{k=0}^K f(\textbf{Y}_k)-K\E(f(\textbf{Y}_k))\right)+ \sup_{f} K\E(f(\textbf{Y}_k)) .$$

By hypothesis, $\sup_{f} K\E(f(\textbf{Y}_k)) \leq K/16 $. Then, we just have to use a classical result of Vapnik-Chervonenkis theory, either in the version of \cite[Theorem 8.3.23]{vershynin_2018}, or of \cite[Corollary 7.18]{Handel_2016}.
There exists a universal constant $C'$ such that

$$\E\left(\sup_{f} \sum_{k=0}^K f(\textbf{Y}_k)-K\E(f(\textbf{Y}_k))\right) \leq C' K \sqrt{ \frac{\VC(\cF)}{K}}.$$

Hence, if $K \geq 256 \ C'^2 \VC(\cF)$, 

$$\E\left(\sup_{f} \sum_{k=0}^K f(\textbf{Y}_k)-K\E(f(\textbf{Y}_k))\right) \leq \frac{K}{16}.$$

Putting everything together, we have the following.
If $C\geq 256  \ C'^2$, with probability $\geq 1-\exp(K/128)$, $\sup_{f} \sum_{k=0}^K f(\textbf{Y}_k) \leq K/16+K/16+K/16$. 
Therefore, by Remark \ref{bravo}, for all $f \in F$

$$\sum_{k=0}^K f(\textbf{X}_k)\leq K/4.$$

\end{proof}

We state a technical lemma that appears in most proofs. Let $g$ be any measurable function  $\R^{d} \rightarrow E $ so that $\E(g(Y_1))$ exists.  We take
$$ \hat a = \argmin_{a \in U } \max_{v \in V} \Med \braket{\frac{1}{m} \sum_{ i \in B_k }g(X_i)-a, v }$$ where $U,V$ are any sets of $E$. We have :
\begin{lemma} \label{general}
If $K \geq C (\VC(V)\vee |\cO|)$ and if $\E(g(Y_1)) \in U$, then, with probability $\geq 1- \exp(-K/126)$, 

$$\max_{v \in V}  \braket{\E(g(Y_1))-\hat{a}, v } \leq  8 \sup_{u \in V} \E\left(  \braket{g(Y_1)-\E(g(Y_1)), u }^2\right)^{1/2} \sqrt{\frac{K}{N}}$$

where $C$ is a universal constant
\end{lemma}

 $\sup_{u \in V} \E\left(  \braket{g(Y_1)-\E(g(Y_1)), u }^2\right)^{1/2}$ is the "weak variance" of the problem.
 
\begin{proof}

Let $K \geq  C (\VC(\cF) \vee |\cO|) $ with $C$ the universal constant from Lemma \ref{main}, let $\bar{g}=\E(g(Y_1)) $ and let 
$$ r_K = 4  \sup_{u \in V} \E\left(  \braket{g(Y_1)-\bar{g}, u }^2\right)^{1/2} \sqrt{\frac{K}{N}}.$$

Let $\cF= \{(\textbf{x}_i)_{i \leq m} \rightarrow \mathbf{1}_{\braket{\frac{1}{m} \sum_i g(x_i) - \E(g(Y_1)), v} \geq r_K} , v \in V \}$. 
The function $f \in \cF$ are compositions of the function $\textbf{x} \rightarrow \frac{1}{m} \sum_i g(x_i) - \E(g(Y_1))$ and of the functions $x \rightarrow \mathbf{1}_{\braket{x, v} \geq r_K} $ for $v \in V$.
The VC-dimension of the set of these compositions is smaller than the VC-dimension of the set of indicator functions indexed by $V$, as recalled in the basic fact 2 at the beginning of Section~\ref{sec:VC}. We just use fact 3 to remove the $r_K$ and we get $ \VC(\cF)\leq c_0  \VC(V)$ for some constant $c_0$.

By Markov's inequality, for any $v \in V$, 

$$\PP(|\braket{\frac{1}{m} \sum_{i\in B_1} g(Y_i) - \bar{g}, v}| \geq r_K) \leq \frac{\E\left( \sum_{i\in B_1} \braket{g(Y_i)-\bar{g}, u }^2\right)}{ m^2 r_K^2} \leq \frac{1}{16}. $$

By Lemma \ref{main}, applied with $\cF$, the following event $\cE$ has probability  $\PP(\cE)\geq 1- \exp(-K/128) $. 

$$ \sup_{v \in V}\Med|\braket{\frac{1}{m} \sum_{i\in B_k} g(X_i) - \bar{g}, v}| \leq r_K$$

For any $a \in U$ if there exists $v^* \in V$ such that $\braket{\bar{g}-a, v^*} > 2r_k$, then, on $\cE$

$$  \Med \braket{\frac{1}{m} \sum_{i\in B_k} g(X_i)-a, v^*} = \braket{\bar{g}-a,v^*}+ \Med \braket{\frac{1}{m} \sum_{i\in B_k} g(X_i)-\bar{g}, v^*}>r_K\geq  \max_{v \in V} \Med\braket{\frac{1}{m} \sum_{i\in B_k} g(X_i)-\bar{g}, v}.$$

 Therefore $a\ne\hat{a}$. As this holds for any $a \in U$ such that  $ \sup_{v \in V} \braket{\bar{g}-a, v} > 2r_k$, it follows that, on $\cE$,
$$ \sup_{v \in V} \braket{\bar{g}-\hat{a}, v} \leq  2 \ r_K.$$
\end{proof}

We can give a somewhat improved version of that lemma: let us note, 

$$ \cR(g, V)= \frac{1}{\sqrt{N}}\E(\sup_{v \in V} \braket{\sum_i \epsilon_i g(Y_i),v}) \ , \ \sigma^2= \sup_{u \in V} \E\left(  \braket{g(Y_1)-\E(g(Y_1)), u }^2\right)$$

$\cR$ is the Rademacher complexity associated to a given problem. The following lemma shows that we can take the best term between the one given by a rescaled Rademacher complexity and the one given by VC-dimension.

\begin{lemmabis}{general}
If $K \gtrsim C ((\VC(V)\wedge (\cR(g, V)/\sigma)^2)\vee |\cO|)$ and if $\E(g(Y_1)) \in U$, then, with probability $\geq 1- \exp(-K/126)$, 

$$\max_{v \in V}  \braket{\E(g(Y_1))-\hat{a}, v } \leq  16 \sigma \sqrt{\frac{K}{N}}$$

where $C$ is a universal constant
\end{lemmabis}

\begin{proof}
We know that this holds when $K \geq C (\VC(V)\vee |\cO|) $.\\

Now if $K\geq C (\cR(g,V)/\sigma)^2\vee |\cO| $, we only need to prove that, for $r_K= 8 \sigma \sqrt{K/N}$
$$ \sup_{v \in V} \sum_k \mathbf{1}_{\braket{\frac{1}{m} \sum_{i \in B_k} g(X_i) - \E(g(Y_1)), v} \geq r_K} \leq K/2 $$

and then we follow the path of the previous proof.

We do this in the classic way, that can be found, for instance in \cite{depersin2019robust} or the supplementary material of \cite{LMSL}

As $K \geq 4 |\cO|$, we only need to show that 
$$\sup_{v \in V} \sum_k \mathbf{1}_{\braket{\frac{1}{m} \sum_{i \in B_k} g(Y_i) - \E(g(Y_1)), v} \geq r_K} \leq K/4 $$

 We define $\phi(t) = 0 $ if $t\leq1/2$, $\phi(t) = 2(t-1/2)$ if $1/2\leq t\leq 1$ and $\phi(t) = 1$ if $t\geq1$. We have $I(t\geq1)\leq \phi(t)\leq I(t\geq1/2)$ for all $t\in\bR$ and so for $v \in V$
  \begin{align*}
 &\sum_{k} I(|\inr{\frac{1}{m} \sum_{i \in B_k} g(Y_i) - \bar{g}, v}|> r_K)\\
 &\leq \sum_{k} I(|\inr{\frac{1}{m} \sum_{i \in B_k} g(Y_i) - \bar{g}, v}|> r_K) - \bP[|\inr{\frac{1}{m} \sum_{i \in B_k} g(Y_i) - \bar{g}, v}|> r_K/2] + \bP[|\inr{\frac{1}{m} \sum_{i \in B_k} g(Y_i) - \bar{g}, v}|> r_K/2]\\
 &\leq \sum_{k}\phi\left(\frac{|\inr{\frac{1}{m} \sum_{i \in B_k} g(Y_i) - \bar{g}, v}|}{r_K}\right) - \bE \phi\left(\frac{|\inr{\frac{1}{m} \sum_{i \in B_k} g(Y_i) - \bar{g}, v}|}{r_K}\right) + \bP[|\inr{\frac{1}{m} \sum_{i \in B_k} g(Y_i) - \bar{g}, v}|> r_K/2]
 \end{align*}For all $v \in V$, we have
 \begin{align*}
 \bP[|\inr{\frac{1}{m} \sum_{i \in B_k} g(Y_i) - \bar{g}, v}|> r/2]\leq \frac{\bE \inr{\frac{1}{m} \sum_{i \in B_k} g(Y_i) - \bar{g}, v}^2}{(r_K/2)^2} \leq \frac{1}{16}
 \end{align*} Next, using the bounded difference inequality (Theorem~6.2 in \cite{MR3185193}), the symmetrization argument and the contraction principle (Chapter~4 in \cite{MR2814399}) -- we refer to the supplementary material of \cite{LMSL} for more details -- with probability at least $1-\exp(-K/128)$, 
 \begin{align*}
  &\sup_{v\in V}\left(\sum_{k}\phi\left(\frac{|\inr{\frac{1}{m} \sum_{i \in B_k} g(Y_i) - \bar{g}, v}|}{r_K}\right) - \bE \phi\left(\frac{|\inr{\frac{1}{m} \sum_{i \in B_k} g(Y_i) - \bar{g}, v}|}{r_K}\right) \right)\\
  & \leq \bE \sup_{v\in V}\left(\sum_{k}\phi\left(\frac{|\inr{\frac{1}{m} \sum_{i \in B_k} g(Y_i) - \bar{g}, v}|}{r_K}\right) - \bE \phi\left(\frac{|\inr{\frac{1}{m} \sum_{i \in B_k} g(Y_i) - \bar{g}, v}|}{r_K}\right) \right)+ \frac{K}{16}\\
  &\leq \frac{4K}{Nr_K} \bE \sup_{v\in V} \inr{v, \sum_{i\in \cup_{k}B_k}\eps_i (g(Y_i)-\bar{g})} +  \frac{K}{16}\\
  & =  \frac{\sqrt{K}}{2\sigma}\bE \braket{\frac{1}{\sqrt{N}}\sum_{i\in \cup_{k}B_k}\eps_i (g(Y_i)-\bar{g}), v} +  \frac{K}{16}\leq \frac{K}{8}
 \end{align*} when $\sqrt{K}\geq  8 \cR(g,V)/\sigma$  or $K \geq 64 (\cR(g,V)/\sigma)^2$

 As a consequence, when $K \geq 64 (\cR(g,V)/\sigma)^2$, with probability at least $1-\exp(-K/126)$, for all $v\in V$, 
 \begin{equation*}
 \sum_{k\in [K]} I(|\inr{\frac{1}{m} \sum_{i \in B_k} g(Y_i) - \bar{g}, v}|> r_K)\leq  \frac{K}{8} + \frac{K}{16}\leq \frac{K}{4}.
 \end{equation*}

\end{proof}

\subsection{Proof of Theorem \ref{main2}, \ref{sparseTheo}, \ref{theo:Cov1}, \ref{theo:Cov2}}

This proofs are very similar: we just apply lemma \ref{general} with the right $g$, $U$ and $V$. We begin with Theorem \ref{main2} for estimating the mean with respect to a general norm.
\begin{proof}[Proof of Theorem \ref{main2}]
We just use lemma \ref{general} with $g : x \rightarrow x $, $U= \R^d$ and $V= \cB^*_0$. We have 
$$\sup_{u \in V}\E\left(  \braket{g(Y_1)-\E(g(Y_1)), u }^2\right)= \sup_{u \in \cB^*_0} \norm{\Sigma u}_2 = \tnorm{\Sigma}$$ and, for any $a \in \R^d$ 
$$ \sup_{v \in \cB^*_0} \braket{\E(Y_1)-a, v}= \norm{\mu-a}$$

so by Lemma \ref{general} , we get that if $K \geq C (\VC(V)\vee |\cO|)$, then, with probability $\geq 1- \exp(-K/126)$

$$ \norm{\hat{\mu}-\mu} \leq 8 \tnorm{\Sigma}\sqrt{\frac{K}{N}} $$
\end{proof}

We continue with the proof of Theorem  \ref{sparseTheo} for estimating sparse means.
\begin{proof}[Proof of Theorem \ref{sparseTheo}]
We just use lemma \ref{general}, this time with $g : x \rightarrow x $, $U= \cU_s$ and $V= \cU_{2s}\cap \cB_2$

We have 
$$\sup_{u \in \cU_{2s}\cap \cB_2}\E\left(  \braket{g(Y_1)-\E(g(Y_1)), u }^2\right)= \sup_{u \in \cU_{2s}\cap \cB_2} \norm{\Sigma^{1/2} u}_2^2 = \lambda_1(\Sigma)$$ and, for any $a \in \cU_s$ (so a fortiori for  $\hat{\mu} \in \cU_s$)
$$ \sup_{v \in \cU_{2s}\cap \cB_2} \braket{\E(Y_1)-a, v}= \norm{\mu-a}_2$$
because we assumed that $\mu \in \cU_s$. So by Lemma \ref{general}, as $\mu \in \cU_s$, we get that if $K \geq C (\VC(\cU_{2s})\vee |\cO|)$, then, with probability $\geq 1- \exp(-K/126)$

$$ \norm{\hat{\mu}-\mu}_2 \leq 8 \lambda_1(\Sigma)\sqrt{\frac{K}{N}} $$

We recalled in part \ref{sec:VC} that   $\VC(\cU_{2s})\leq 2s\log(d/s)$, which concludes the proof.

\end{proof}

We move to the proof of Theorem \ref{theo:Cov1}, for estimating covariance with respect to the canonical euclidean operator norm.

\begin{proof}[Proof of Theorem \ref{theo:Cov1}]
This time, we take $g: x \rightarrow x x^T$, $U= \Sym(d)$, and $V=\{u u^T | u \in \cB_2(\R^d)\}$. We notice that $\E(g(Y_1))= \Sigma$ \\

We have 
$$\sup_{M \in V}\E\left(  \braket{g(Y_1)-\E(g(Y_1)), M }^2\right)= \sigma^2 $$ 

by definition of $\sigma^2$, and for any $A \in \Sym(d)$ (so a fortiori for  $\hat{\Sigma} \in \Sym(d)$)
$$ \sup_{M \in V} \braket{\Sigma-A, M}= \tnorm{\Sigma-A}$$

So by Lemma \ref{general}, as $\Sigma \in \Sym(d)$, we get that if $K \geq C (\VC(V)\vee |\cO|)$, then, with probability $\geq 1- \exp(-K/126)$

$$ \tnorm{\hat{\Sigma}-\Sigma}\leq 8 \beta \sqrt{\frac{K}{N}} $$

We recalled in part \ref{sec:VC} (Proposition \ref{important}) that   $\VC(V)\leq c_0 d$, for some universal constant $c_0$, which concludes the proof.
 \end{proof}

 \begin{proof}[Proof of Theorem \ref{theo:Cov2}]
We take $g: x \rightarrow x x^T$, $U= \cM_d^{r}(\R)$, and $V=\{{ U \in \cM_d^{2r}(\R), \ \norm{U}_F=1}\}$. \\

Let $U\in \cM_d^{2r}(\R)$ be such that $\norm{U}_F=1$. 
Then, one can write $U= \sum \lambda_i x_i x_i^T$ with $x_i \in \cB_2$ and $\sum \lambda_i=1$. 
By convexity,
$$\E(\braket{( Y_1 Y_1^T-\Sigma),U}_F^2 \leq \sum \lambda_i \E(\braket{( Y_1 Y_1^T-\Sigma),x_i x_i^T}_F^2  \leq \sum \lambda_i \sigma^2 \leq \sigma^2.$$

and for any $A \in \cM_d^{r}(\R)$ (so a fortiori for  $\hat{\Sigma} \in \cM_d^{r}(\R)$), taking $M= (\Sigma-A)/\norm{\Sigma-A}_F$ we notice that
$$ \sup_{M \in V} \braket{\Sigma-A, M}\geq \norm{\Sigma-A}_F$$

So by Lemma \ref{general}, as $\Sigma \in \cM_d^{r}(\R)$, we get that if $K \geq C (\VC(V)\vee |\cO|)$, then, with probability $\geq 1- \exp(-K/126)$

$$ \norm{\hat{\Sigma}-\Sigma}_F\leq 8 \beta \sqrt{\frac{K}{N}} $$

We recalled in part \ref{sec:VC} (Proposition \ref{important}) that   $\VC(V)\leq c_0 k d$, for some universal constant $c_0$, which concludes the proof.
 \end{proof}

\subsection{Proof of Theorem \ref{theo:Reg}}

This proof is a bit different from the rest because we will have to control two different events.

\begin{proof}

Let $\cF= \{(\textbf{x},\textbf{y}) \in \R^{(d+1)\times m} \rightarrow \mathbf{1}_{\braket{u , \sum_{i }(y_i-\braket{\beta^* , x_i})x_i}^2 \geq m^2 r^2} , u \in \Ba \}$. This is not a set of indicators of half-spaces, but $\cF$ is the composition of $g :(\textbf{x},\textbf{y}) \in R^{(d+1) \times m} \rightarrow (u \rightarrow \braket{u , \sum_{i }(y_i-\braket{\beta^* , x_i})x_i}^2-m^2 r^2 ) \in \R^d_2[X]  $ and of $\{P \in \R^d_2[X] \rightarrow 1_{P(u)\geq 0}, u \in \cS-\cS \}$. By  Proposition \ref{important} (when $ \cS= \R^d$) and Proposition~\ref{lessimportant} (when $\cS=\cU_{s}$) , there exists an absolute constant $c$ such that $\VC(\cF) \leq \VC(\cS-\cS)$.

Let $\cG =\{(\textbf{x}_i) \rightarrow \mathbf{1}_{\sum \braket{x_i,u}^2\geq \tilde r} , u \in \cS-\cS \}$. 
The same way, by Proposition \ref{important} (when $ \cS= \R^d$) and Proposition~\ref{lessimportant} (when $\cS=\cU_{s}$), there exists an absolute constant $c$ such that $\VC(\cG) \leq c \VC(\cS)$. 
Assume that $K \geq  C (\VC(\cF) \vee \VC(\cG) \vee |\cO|) $ where $C$ is the universal constant introduced in Lemma \ref{main}.

\textbf{Multiplier process} : Let 
$$ r= 4 \sqrt{ \frac{\sup_{u \in \Ba} \E(\xi_1^2 \braket{u , Y_1}^2)K}{N}}. $$
For all $u \in \Ba$,
 
 $$ \PP\left(\frac{1}{m}|\sum_{i \in B_1}(V_i-\braket{\beta^* , Y_i})\braket{u , Y_i}| \geq r \right) \leq \frac{\E(\xi_1^2 \braket{u , Y_1}^2)}{m r^2} \leq \frac{1}{16}. $$
 
 By Lemma~\ref{main} applied with $\cF$, it follows that the following event $\cE$ has probability $\geq 1 - \exp(-K/128) $: for all $u \in \Ba$, there exist more than $3/4K$ blocks $k$ where
 
  $$  |\sum_{i \in B_k}(Z_i-\braket{a , X_i})\braket{u , X_i}| \leq  m r. $$

 \textbf{Quadratic process} : From Chebyshev's inequality, for any $u\in \cS-\cS $,
 $$\PP\left(\frac{1}{m}\sum_{i \in B_1}|\braket{u , Y_i}| \leq \E|\braket{u , Y_1}|- 4 \sqrt{\frac{ \E\braket{u , Y_1}^2 }{m}} \right)\leq  \frac{1}{16}.$$
 
 So, when $K \leq \gamma^2 N/64$, by the small ball hypothesis, 
 $$\PP\left(\frac{1}{m}\sum_{i \in B_1}|\braket{u , Y_i}| \leq \gamma/2 \sqrt{\E\braket{u , Y_1}^2 } \right)\leq  \frac{1}{16}.$$ 
 As $\frac{1}{m}\sum_{i \in B_k}|\braket{u , X_i}| \leq \sqrt{ \frac{1}{m}\sum_{i \in B_k}|\braket{u , X_i}|^2} $, by Lemma~\ref{main} applied with $\cG$ and $\tilde r =  m \gamma^2/4 \E\braket{u , Y_1}^2  $, the following event $\cA$ has probability probability $\geq 1 - \exp(-K/128) $: for all $u \in \cS-\cS $, there exists more than $3/4K$ blocks $k$ where
 
$$ \frac{1}{m}\sum_{i \in B_k}|\braket{u , X_i}|^2 \geq \gamma^2/4 \braket{u,\Sigma u} .$$
 
So we have $\cQ_{1/4}^k \frac{1}{m}\sum_{i \in B_k}|\braket{u , X_i}|^2 \geq \gamma^2/4 \times \braket{u,\Sigma u} $. \\

 \textbf{Conclusion of the proof.} The event $\cE\cap \cA$ has probability at least $1-2\exp(-K/128)$. 

On $\cA$, if $u \in \hat{B}_\Sigma$, then $\braket{u,\Sigma u} \leq 4/ \gamma^2$, so, on $\cA\cap \cE$,
$$\max_{u\in \hat B_\Sigma} \Med_k \sum_{i \in B_k}(Z_i-\braket{\beta^* , X_i})\braket{u , X_i}\leq 2/\gamma \max_{u\in  B_\Sigma} \Med_k \sum_{i \in B_k}(Z_i-\braket{\beta^* , X_i})\braket{u , X_i}  \leq 2m r/\gamma . $$

For any $\beta \in \cS$ such that $\Sigma (\beta-\beta^*)\neq 0$, let $$u^*=\frac{\beta-\beta^*}{\sqrt{\cQ_{1/4}^k \frac{1}{m}\sum_{i \in B_k}|\braket{\beta-\beta^* , X_i}|^2}}.
$$
By construction $u^* \in \hat{B}_\Sigma $, so for 3/4 of the blocks, on $\cE\cap\cA$,
$$  |\sum_{i \in B_k}(Z_i-\braket{\beta^* , X_i})\braket{u^* , X_i}| \leq 2 m r / \gamma .$$

On the other hand, by definition,  for 3/4 of the blocks,
\[
\frac{1}{m}\sum_{i \in B_{k} }|\braket{\beta-\beta^* , X_i}|^2 \geq  \cQ_{1/4}^{\tilde{k}} \frac{1}{m}\sum_{i \in B_{\tilde{k}} }|\braket{\beta-\beta^* , X_i}|^2.
\]
Therefore, for at least half the blocks, both inequalities hold, so, on $\cE\cap\cA$,

  \begin{align*}
 \sum_{i \in B_k}(Z_i-\braket{\beta-\beta^*+\beta^* , X_i})\braket{u^* , X_i} & \geq -  2 m r / \gamma + m \sqrt{\cQ_{1/4}^{\tilde{k}} \frac{1}{m}\sum_{i \in B_{\tilde{k}} }|\braket{\beta-\beta^* , X_i}|^2}\\
 &\geq -  2 m r /\gamma+ m\gamma/2  \sqrt{(\beta-\beta^*) \Sigma (\beta-\beta^*)}.  
   \end{align*} 

It follows that, on $\cE\cap\cA$, if  $ \sqrt{(\beta-\beta^*) \Sigma (\beta-\beta^*)}  \geq 8r/\gamma^2$, then $\sum_{i \in B_k}(Z_i-\braket{\beta , X_i}) \geq - 2mr/ \gamma + 4 mr/\gamma \geq \sum_{i \in B_k}(Z_i-\braket{\beta^* , X_i}) $  and $\beta$ can not be the chosen estimator. This concludes the proof.
 \end{proof}

\subsection{Proof of Proposition \ref{prop:BetterRB1}} \label{preuve:RB1}

\begin{proof}
We study separately what happens on each subspace $E_i$. 
The dimension of $E_i$ is $d_i$ and the orthogonal projection $\mu_i$ of $\mu$ is $s$-sparse  on the set of vectors $ u_1^i, u_2^i, ... , u_{d}^i$.
$\mu_i$ is also generated by $e_{s_{i-1}+1}, ..., e_{s_{i}}$ which is a base of $E_i$. 
We choose which representation of $\mu_i$ leads to the best bound: if $d_i\geq s \log(d/s)$, we choose the first, else we choose the second. 
The preliminary bound holds if $K_i$ is larger than either $d_i$ or $s \log(d/s)$.
Let $\tilde{\mu}_i$ denote our estimation on $E_i$:
\[
\tilde \mu_i= \texttt{proc}(X_1^i, ..., X_n^i, K_i, e_{s_{i-1}+1}, ..., e_{s_{i}}, d_i)\textbf{1}_{d_i < s \log(d/s)} + \texttt{proc}(X_1^i, ..., X_n^i, K_i, u_1^i, u_2^i, ...,  u_{d}^i ,s) \textbf{1}_{d_i \geq s \log(d/s)}.
\]

If $K_i \geq (C d_i \wedge s \log(d/s)) \vee |\cO|  $, on an event $\cE_i$ of probability $\geq 1-\exp(-K_i/128)$,
$$\|\tilde{\mu}_i-\mu_i\|_2 \leq 8  \sqrt{\frac{\tnorm{\tilde \Sigma_i} K_i}{N}}= \sqrt{\frac{\lambda_1 K}{N}}. $$

Let $\cE= \cap \cE_i$, so $\PP(\cE)\geq 1-\sum \exp(- 2^i K/128) \geq 1- 2 \exp(- K/128)$ if both $K\geq 128$ and $K \geq  2^{-i} C d_i \wedge s \log(d/s)$ for all $i$.  
As the subspaces $E_i$ are orthogonal to each other (as eigenspaces of a symmetric matrix), by Pythagoras theorem,
$$ \|\tilde{\mu}-\mu\|_2^2 = \sum_i \|\tilde{\mu}_i-\mu_i\|_2^2 \leq \log(d) \sqrt{\frac{\lambda_1 K}{N}}.$$

\end{proof}

\vspace{0.7cm}
\textbf{Acknowlegements:} I would like to thank Guillaume Lecué and Matthieu Lerasle for their precious comments.





\begin{footnotesize}
\bibliographystyle{plain}
\bibliography{biblio}

\begin{thebibliography}{10}

\bibitem{JMLR:v20:17-504}
Mehmet~Eren Ahsen and Mathukumalli Vidyasagar.
\newblock An approach to one-bit compressed sensing based on probably
  approximately correct learning theory.
\newblock {\em Journal of Machine Learning Research}, 20(11):1--23, 2019.

\bibitem{MR1688610}
Noga Alon, Yossi Matias, and Mario Szegedy.
\newblock The space complexity of approximating the frequency moments.
\newblock {\em J. Comput. System Sci.}, 58(1, part 2):137--147, 1999.
\newblock Twenty-eighth Annual ACM Symposium on the Theory of Computing
  (Philadelphia, PA, 1996).

\bibitem{audibert2011}
Jean-Yves Audibert and Olivier Catoni.
\newblock Robust linear least squares regression.
\newblock {\em Ann. Statist.}, 39(5):2766--2794, 10 2011.

\bibitem{MR762855}
Lucien Birg\'{e}.
\newblock Stabilit\'{e} et instabilit\'{e} du risque minimax pour des variables
  ind\'{e}pendantes \'{e}quidistribu\'{e}es.
\newblock {\em Ann. Inst. H. Poincar\'{e} Probab. Statist.}, 20(3):201--223,
  1984.

\bibitem{MR3185193}
St\'{e}phane Boucheron, G\'{a}bor Lugosi, and Pascal Massart.
\newblock {\em Concentration inequalities}.
\newblock Oxford University Press, Oxford, 2013.
\newblock A nonasymptotic theory of independence, With a foreword by Michel
  Ledoux.

\bibitem{MR3124669}
S\'{e}bastien Bubeck, Nicol\`o Cesa-Bianchi, and G\'{a}bor Lugosi.
\newblock Bandits with heavy tail.
\newblock {\em IEEE Trans. Inform. Theory}, 59(11):7711--7717, 2013.

\bibitem{AIHPB_2012__48_4_1148_0}
Olivier Catoni.
\newblock Challenging the empirical mean and empirical variance: A deviation
  study.
\newblock {\em Annales de l'I.H.P. Probabilit\'es et statistiques},
  48(4):1148--1185, 2012.

\bibitem{MR3845006}
Mengjie Chen, Chao Gao, and Zhao Ren.
\newblock Robust covariance and scatter matrix estimation under {H}uber's
  contamination model.
\newblock {\em Ann. Statist.}, 46(5):1932--1960, 2018.

\bibitem{MR3909640}
Yu~Cheng, Ilias Diakonikolas, and Rong Ge.
\newblock High-dimensional robust mean estimation in nearly-linear time.
\newblock In {\em Proceedings of the {T}hirtieth {A}nnual {ACM}-{SIAM}
  {S}ymposium on {D}iscrete {A}lgorithms}, pages 2755--2771. SIAM,
  Philadelphia, PA, 2019.

\bibitem{Bartlett19}
Yeshwanth Cherapanamjeri, Nicolas Flammarion, and Peter~L. Bartlett.
\newblock Fast mean estimation with sub-gaussian rates, 2019.

\bibitem{cherapanamjeri2019algorithms}
Yeshwanth Cherapanamjeri, Samuel~B. Hopkins, Tarun Kathuria, Prasad
  Raghavendra, and Nilesh Tripuraneni.
\newblock Algorithms for heavy-tailed statistics: Regression, covariance
  estimation, and beyond, 2019.

\bibitem{chinot2018statistical}
Geoffrey Chinot, Lecu{\'e} Guillaume, and Lerasle Matthieu.
\newblock Statistical learning with lipschitz and convex loss functions.
\newblock {\em arXiv preprint arXiv:1810.01090}, 2018.

\bibitem{depersin2019robust}
Jules Depersin and Guillaume Lecué.
\newblock Robust subgaussian estimation of a mean vector in nearly linear time,
  2019.

\bibitem{MR3576558}
Luc Devroye, Matthieu Lerasle, Gabor Lugosi, and Roberto~I. Oliveira.
\newblock Sub-{G}aussian mean estimators.
\newblock {\em Ann. Statist.}, 44(6):2695--2725, 2016.

\bibitem{MR3631028}
Ilias Diakonikolas, Gautam Kamath, Daniel~M. Kane, Jerry Li, Ankur Moitra, and
  Alistair Stewart.
\newblock Robust estimators in high dimensions without the computational
  intractability.
\newblock In {\em 57th {A}nnual {IEEE} {S}ymposium on {F}oundations of
  {C}omputer {S}cience---{FOCS} 2016}, pages 655--664. IEEE Computer Soc., Los
  Alamitos, CA, 2016.

\bibitem{dudley1978}
R.~M. Dudley.
\newblock Central limit theorems for empirical measures.
\newblock {\em Ann. Probab.}, 6(6):899--929, 12 1978.

\bibitem{guillaume2017learning}
Lecué Guillaume and Lerasle Matthieu.
\newblock Learning from mom's principles: Le cam's approach, 2017.

\bibitem{MR0301858}
Frank~R. Hampel.
\newblock A general qualitative definition of robustness.
\newblock {\em Ann. Math. Statist.}, 42:1887--1896, 1971.

\bibitem{MR0359096}
Frank~R. Hampel.
\newblock Robust estimation: a condensed partial survey.
\newblock {\em Z. Wahrscheinlichkeitstheorie und Verw. Gebiete}, 27:87--104,
  1973.

\bibitem{hopkins2018sub}
Samuel~B Hopkins.
\newblock Sub-gaussian mean estimation in polynomial time.
\newblock {\em arXiv preprint arXiv:1809.07425}, 2018.

\bibitem{JMLR:v17:14-273}
Daniel Hsu and Sivan Sabato.
\newblock Loss minimization and parameter estimation with heavy tails.
\newblock {\em Journal of Machine Learning Research}, 17(18):1--40, 2016.

\bibitem{MR0161415}
Peter~J. Huber.
\newblock Robust estimation of a location parameter.
\newblock {\em Ann. Math. Statist.}, 35:73--101, 1964.

\bibitem{MR2488795}
Peter~J. Huber and Elvezio~M. Ronchetti.
\newblock {\em Robust statistics}.
\newblock Wiley Series in Probability and Statistics. John Wiley \& Sons, Inc.,
  Hoboken, NJ, second edition, 2009.

\bibitem{MR855970}
Mark~R. Jerrum, Leslie~G. Valiant, and Vijay~V. Vazirani.
\newblock Random generation of combinatorial structures from a uniform
  distribution.
\newblock {\em Theoret. Comput. Sci.}, 43(2-3):169--188, 1986.

\bibitem{lecu2017robust}
Guillaume Lecué and Matthieu Lerasle.
\newblock Robust machine learning by median-of-means : theory and practice,
  2017.

\bibitem{lecu2016regularization}
Guillaume Lecué and Shahar Mendelson.
\newblock Regularization and the small-ball method i: sparse recovery, 2016.

\bibitem{Led01}
Michel Ledoux.
\newblock {\em The concentration of measure phenomenon}, volume~89 of {\em
  Mathematical Surveys and Monographs}.
\newblock American Mathematical Society, Providence, RI, 2001.

\bibitem{MR2814399}
Michel Ledoux and Michel Talagrand.
\newblock {\em Probability in {B}anach spaces}.
\newblock Classics in Mathematics. Springer-Verlag, Berlin, 2011.
\newblock Isoperimetry and processes, Reprint of the 1991 edition.

\bibitem{LO}
M.~Lerasle and R.~Oliveira.
\newblock Robust empirical mean estimators.
\newblock Technical report, IMPA and CNRS, 2011.

\bibitem{lerasle2019lecture}
Matthieu Lerasle.
\newblock Lecture notes: Selected topics on robust statistical learning theory,
  2019.

\bibitem{li2017robust}
Jerry Li.
\newblock Robust sparse estimation tasks in high dimensions, 2017.

\bibitem{lugosi2019mean}
Gabor Lugosi and Shahar Mendelson.
\newblock Mean estimation and regression under heavy-tailed distributions--a
  survey, 2019.

\bibitem{lugosi2019sub}
G{\'a}bor Lugosi, Shahar Mendelson, et~al.
\newblock Sub-gaussian estimators of the mean of a random vector.
\newblock {\em The Annals of Statistics}, 47(2):783--794, 2019.

\bibitem{lugosi2017regularization}
Gábor Lugosi and Shahar Mendelson.
\newblock Regularization, sparse recovery, and median-of-means tournaments,
  2017.

\bibitem{lugosi2018nearoptimal}
Gábor Lugosi and Shahar Mendelson.
\newblock Near-optimal mean estimators with respect to general norms, 2018.

\bibitem{LMSL}
Z.~Szabo M.~Lerasle, T.~Matthieu and G.~Lecué.
\newblock Monk – outliers-robust mean embedding estimation by
  median-of-means.
\newblock Technical report, CNRS, University of Paris 11, Ecole Polytechnique
  and CREST, 2017.

\bibitem{mendelson2017extending}
Shahar Mendelson.
\newblock Extending the small-ball method, 2017.

\bibitem{mendelson2018robust}
Shahar Mendelson and Nikita Zhivotovskiy.
\newblock Robust covariance estimation under $l_4-l_2$ norm equivalence, 2018.

\bibitem{MS}
S~Minsker and N.~Strawn.
\newblock Distributed statistical estimation and rates of convergence in normal
  approximation.
\newblock Technical report, arXiv: 1704.02658, 2017.

\bibitem{minsker2015geometric}
Stanislav Minsker.
\newblock Geometric median and robust estimation in banach spaces.
\newblock {\em Bernoulli}, 21(4):2308--2335, 2015.

\bibitem{MR702836}
A.~S. Nemirovsky and D.~B.~and Yudin.
\newblock {\em Problem complexity and method efficiency in optimization}.
\newblock A Wiley-Interscience Publication. John Wiley \& Sons, Inc., New York,
  1983.
\newblock Translated from the Russian and with a preface by E. R. Dawson,
  Wiley-Interscience Series in Discrete Mathematics.

\bibitem{prasad2019unified}
Adarsh Prasad, Sivaraman Balakrishnan, and Pradeep Ravikumar.
\newblock A unified approach to robust mean estimation, 2019.

\bibitem{prasadFDP}
Adarsh Prasad, Sivaraman Balakrishnan, and Pradeep Ravikumar.
\newblock A robust univariate mean estimator is all you need., 2020.

\bibitem{SAUER1972145}
N~Sauer.
\newblock On the density of families of sets.
\newblock {\em Journal of Combinatorial Theory, Series A}, 13(1):145 -- 147,
  1972.

\bibitem{MR0133937}
John~W. Tukey.
\newblock The future of data analysis.
\newblock {\em Ann. Math. Statist.}, 33:1--67, 1962.

\bibitem{vandervaart}
Aad van~der Vaart and Jon~A. Wellner.
\newblock {\em A note on bounds for VC dimensions}, volume Volume 5 of {\em
  Collections}, pages 103--107.
\newblock Institute of Mathematical Statistics, Beachwood, Ohio, USA, 2009.

\bibitem{Handel_2016}
R.~van Handel.
\newblock {\em Probability in high dimension}.
\newblock 2016.

\bibitem{vershynin_2018}
Roman Vershynin.
\newblock {\em High-Dimensional Probability: An Introduction with Applications
  in Data Science}.
\newblock Cambridge Series in Statistical and Probabilistic Mathematics.
  Cambridge University Press, 2018.

\bibitem{10.5555/524030}
M.~Vidyasagar.
\newblock {\em A Theory of Learning and Generalization: With Applications to
  Neural Networks and Control Systems}.
\newblock Springer-Verlag, Berlin, Heidelberg, 1997.

\bibitem{Voro18}
Yevgeniy Vorobeychik and Murat Kantarcioglu.
\newblock Adversarial machine learning.
\newblock {\em Synthesis Lectures on Artificial Intelligence and Machine
  Learning}, 12:1--169, 08 2018.

\bibitem{10.2307/1994937}
Hugh~E. Warren.
\newblock Lower bounds for approximation by nonlinear manifolds.
\newblock {\em Transactions of the American Mathematical Society},
  133(1):167--178, 1968.

\bibitem{Wei2017EstimationOT}
Xiaohan Wei and Stanislav Minsker.
\newblock Estimation of the covariance structure of heavy-tailed distributions.
\newblock In {\em NIPS}, 2017.

\end{thebibliography}
\end{footnotesize}

\end{document}